\documentclass[11pt]{article}
\usepackage[utf8]{inputenc}
\usepackage[margin=2cm,top=3cm,bottom=3cm]{geometry}
\usepackage{hyperref}
\usepackage{makecell}
\usepackage{subcaption}

\usepackage{graphicx}
\usepackage{epsfig}
\usepackage{url}
\usepackage{color}
\usepackage{trivfloat}

\usepackage{subcaption}
\usepackage{amsmath,amstext,amsbsy,amsopn,amscd,amsxtra,upref,amssymb,amsthm}
\usepackage{bm}
\usepackage{grffile}
\usepackage{algorithm}
\usepackage{algpseudocode}

\newcommand{\EE}{\mathbb{E}}
\newcommand{\PP}[1]{\mathbb{P}\left\{#1\right\}}

\newcommand{\trace}{\mathsf{tr}}

\renewcommand{\Vec}{\operatorname{vec}}

\newcommand{\normal}{\mathcal{N}}

\newcommand{\redvel}{\mathcal{O}}

\newcommand{\R}{\mathbb{R}}

\newcommand{\lam}{\lambda}

\newcommand{\eps}{\varepsilon}

\newcommand{\mb}[1]{\left[\begin{array}{#1}}
\newcommand{\me}{\end{array}\right]}
\newcommand{\smb}{\left[\begin{smallmatrix}}
\newcommand{\sme}{\end{smallmatrix}\right]}

\newcommand{\kron}{\otimes}

\newcommand{\Xmpow}[1]{X^{[#1]}}

\newcommand{\Ompow}[1]{\Omega^{[#1]}}

\newcommand{\Omle}[1]{\Omega^{\leq #1}}
\newcommand{\Omlek}{\Omega^{\leq k}}

\newcommand{\ple}[1]{\Pi^{\leq #1}}
\newcommand{\tildeple}[1]{\tilde{\Pi}^{\leq #1}}
\newcommand{\plek}{\Pi^{\leq k}}
\newcommand{\mi}{\mathcal{I}}
\newcommand{\mj}{\mathcal{J}}
\newcommand{\mile}[1]{\mathcal{I}_{\leq #1}}

\newcommand{\milek}{\mathcal{I}_{\leq k}}
\newcommand{\migk}{\mathcal{I}_{>k}}
\newcommand{\mig}[1]{\mathcal{I}_{>#1}}

\newcommand{\lCeil}{%
  \mathopen{%
    \vbox{%
      \offinterlineskip
      \halign{##\cr
        \hbox{\hskip 0.45ex \vrule height 0.4pt width 0.5ex}\cr   
        \noalign{\kern-1.5pt}
        \hbox{$\lceil$}\cr                           
      }%
    }%
  }%
}

\newcommand{\rCeil}{%
  \mathclose{%
    \vbox{%
      \offinterlineskip
      \halign{##\cr
        \hbox{\hskip 0.05ex \vrule height 0.4pt width 0.5ex}\cr   
        \noalign{\kern-1.5pt}
        \hbox{$\rceil$}\cr                           
      }%
    }%
  }%
}

\newcommand{\pinv}{\dagger}

\DeclareMathOperator{\Var}{Var}

\newtheorem{theorem}{Theorem}

\newtheorem{lemma}[theorem]{Lemma}
\newtheorem{proposition}[theorem]{Proposition}
\newtheorem{remark}[theorem]{Remark}
\newtheorem{corollary}[theorem]{Corollary}
\newtheorem{definition}[theorem]{Definition}

\title{Stochastic trace estimation with tensor train random vectors}

\author{Zvonimir Bujanovi\'{c}\footnote{University of Zagreb, Faculty of Science, Department of Mathematics, Croatia. {\tt zbujanov@math.hr}} \and Daniel Kressner\footnote{Institute of Mathematics, EPFL, Switzerland. {\tt daniel.kressner@epfl.ch}} \and Hrvoje Oli\'{c}\footnote{University of Zagreb, Faculty of Science, Department of Mathematics, Croatia. {\tt hrvoje.olic@math.hr}.}}

\begin{document}

\maketitle

\begin{abstract}
    Stochastic trace estimation is a standard tool for approximating the trace of a large-scale matrix available only through matrix-vector products. However, in tensor-structured settings, unstructured Gaussian or Rademacher test vectors may be prohibitively expensive to store and compute with, while cheaper rank-one tensor-product vectors can require sample complexities that grow exponentially with the tensor order. This work studies Gaussian random tensor train vectors as a structured alternative for stochastic trace estimation. We show that, with a suitable choice of the tensor train rank, random tensor train vectors recover dimension-independent guarantees for the Girard--Hutchinson estimator. In particular, a median-of-means variant with tensor train rank $r \geq d-1$ achieves the same dependence on the accuracy $\eps$ and failure probability $\delta$ as the classical estimator based on unstructured Gaussian vectors. We further prove an oblivious subspace injection result for sketches formed from independent Gaussian random tensor train vectors: tensor train rank $r\geq d-1$ and $\redvel(\eps^{-2}(k+\log(1/\delta)))$ samples suffice for a $k$-dimensional target subspace. Finally, we investigate the use of such sketches within the Nystr\"{o}m++ framework. We show that the resulting estimator can achieve the desired $\redvel(\eps^{-1})$ sample complexity under an additional spectral-tail condition. These results provide clarififcation on both the potential and the limitations of random tensor train vectors in stochastic trace estimation.
\end{abstract}

\section{Introduction}
\label{sec:intro}

The need for estimating the trace of a large-scale, symmetric positive semidefinite matrix $A \in \R^{N \times N}$ arises in a wide range of applications \cite{BaiFaheyGolub96, CortinovisKressner22,  EdenEtAl17, LinSaadYang16, UbaruChenSaad17, UbaruSaad18, Weisse06}. A common setting in these applications is that $A$ is only available through matrix-vector products $x \mapsto Ax$. Obviously, one could  simply evaluate all diagonal elements $A_{ii} = e_i^T (A e_i)$, by applying $A$ to all $N$ unit vectors $e_i$, and compute the trace exactly, but the $N$ matrix-vector products required for such a procedure render the computation prohibitively expensive. Instead, a stochastic approach proposed by Girard \cite{Girard89} and Hutchinson \cite{Hutchinson89} is preferable:
\begin{equation}
    \label{eq:GH}
    \trace(A) \approx H_{\ell}(A) := \frac{1}{\ell} \sum_{i=1}^\ell \omega_i^T A \omega_i.
\end{equation}
Here $\omega_1, \ldots, \omega_{\ell} \in \R^N$ are independent, identically distributed random vectors. For isotropic random vectors (i.e., $\EE[\omega_i \omega_i^T] = I_N$), it is well known that $H_{\ell}(A)$ is an unbiased estimator for $\trace(A)$: $\EE[H_{\ell}(A)] = \trace(A)$. The theoretical properties of the Girard--Hutchinson estimator \eqref{eq:GH} have been studied intensively, especially in the case when the vectors $\omega_i$ are Rademacher or Gaussian random vectors; see, e.g.,~\cite{AvronToledo11,CortinovisKressner22,RoostaAscher15}. In these cases, it is well known that a sample complexity of
$\ell = \redvel(\eps^{-2} \log(1/\delta))$ is sufficient to ensure
\begin{equation}
    \label{eq:trace-eps-delta}
    \PP{|\trace(A) - H_{\ell}(A)| \leq \eps \, \trace(A)} \geq 1-\delta.
\end{equation}
Variance reduction through randomized low-rank approximation can be used to reduce the dependence on $\eps$ from $\eps^{-2}$ to $\eps^{-1}$, as manifested by the Hutch++ \cite{MeyerMuscoMusco21} and Nystr\"om++~\cite{PerssonCortinovisKressner22} algorithms.

In this work, we consider the case when $N = n^d$ and, hence, $A$ can be viewed as a linear operator on tensors of order $d$. As $d$ increases, it quickly becomes
computationally infeasible to generate unstructured random vectors of length $n^d$ and apply $A$. On the other hand, in these cases, $A$ is usually highly structured and its 
application to vectors in a compatible tensor format can be executed efficiently. For example, consider the simple case of a Kronecker product: $A = A^{(1)} \kron A^{(2)} \kron \ldots \kron A^{(d)}$ with $A^{(i)} \in \R^{n \times n}$. Generating a random Gaussian vector $\omega \in \R^N$ and computing $A\omega$ would be rather inefficient. This is very different when considering a random \emph{rank-one tensor} of order $d$, i.e.,  $\omega = \omega^{(1)} \kron \omega^{(2)} \kron \ldots \kron \omega^{(d)}$ with small random vectors $\omega^{(j)} \in \R^n$.
Then the computation
$$
    \omega^T A\omega = ((\omega^{(1)})^T A^{(1)} \omega^{(1)}) \cdot ((\omega^{(2)})^T A^{(2)} \omega^{(2)}) \cdot \ldots \cdot ((\omega^{(d)})^T A^{(d)} \omega^{(d)})
$$
only involves small vectors, reducing the complexity from exponential to linear in $d$. Nevertheless, this comes at a steep price. Recent work \cite{AvronMeyer25} shows that random rank-one tensors may have poor concentration for trace estimation within the Girard--Hutchinson framework, with the sample complexity $\ell$ growing exponentially in the tensor order $d$ to attain fixed accuracy and failure probability. Similar exponential dependencies arise when collections of such vectors are used as sketching matrices in, e.g., the randomized SVD algorithm \cite{halko2011finding, CamanoEpperlyMeyerTropp25}, limiting their use to small values of $d$; see~\cite{BujanovicKressner21,BujanovicGrubisicKressnerLam25}. The work~\cite{ZhangKileel25} proposes median sketches, inspired by median-of-means, to obtain dimension-reduction guarantees for structured sets using tensor-structured or sparse measurements; however, the resulting sketch is not a linear mapping.

The limitations of random rank-one vectors raise the question of whether a richer tensor format can recover dimension-independent guarantees for stochastic trace estimation and sketching. Several alternative distributions of random vectors that are compatible with operators with tensor structure have been studied recently. Among them, the tensor train (TT) format \cite{Oseledets11} is particularly attractive: it has a low memory footprint, permits efficient computations with compatible operators, and offers a tunable tradeoff between computational cost and expressiveness through the TT rank $r$. Randomized algorithms based on random TT vectors have recently been successfully applied, especially in sketching methods for low-rank approximation and TT rounding \cite{DaasEtAl23,HuberSchneiderWolf17,BucciPallitaRobol25,KressnerVandereyckenVoorhaar23}. However, for a long time, the available theoretical guarantees have been elusive or, at least, rather pessimistic. A very recent contribution of Cazeaux et al.~\cite{CazeauxDupuyJustiniano26} significantly improves this situation by providing analysis for so called TTStack/BSTT sketches, variants of TT structured random sketching. It is shown that TTStack can satisfy oblivious subspace embedding and injection properties (needed, e.g., with the randomized SVD) avoiding the exponential dependence on the tensor order $d$ that appears for random rank-one tensor sketches. Instead, $d$ enters only linearly through the TT rank.

In this work, we discuss the use of random vectors with TT structure for the purpose of stochastic trace estimation. We show that, unlike the random rank-one tensors, the number of samples needed to achieve \eqref{eq:trace-eps-delta} does not depend on the tensor order $d$ if the TT rank $r$ is chosen appropriately. 
In particular, when using the median-of-means technique in combination with Girard--Hutchinson estimator, the TT rank of $r \geq d-1$ suffices to obtain a sample complexity of order $\redvel(\eps^{-2} \log(1/\delta))$, matching the dependence on $\eps$ and $\delta$ of the classical unstructured Gaussian estimator. 

We also prove an oblivious subspace injection result for Gaussian TT sketches, complementing the comparison in \cite[Table 1]{CazeauxDupuyJustiniano26}. More precisely, we show that a sketch consisting of $\redvel(\eps^{-2}(k + \log(1/\delta)))$ independent random TT vectors of rank $r \geq d-1$ satisfies a so-called $(1-\eps, \delta, k)$ oblivious subspace injection (OSI) property. This removes a linear dependence on $d$ from the sketch dimension compared with the corresponding TTStack guarantee, although $d$ still enters through the TT rank of the random vectors. Finally, we investigate whether these OSI guarantees are sufficient to use random TT vectors within the Nystr\"{o}m++ framework and thereby obtain an $\redvel(\eps^{-1})$-sample trace estimator. We show that this is the case under a spectral decay assumption.

The rest of this paper is structured as follows: In Section \ref{sec:prelim}, we recall the notion of Gaussian random TT vectors and establish their basic properties. Section~\ref{sec:gh} then discusses the sample complexity for Girard--Hutchinson estimators and its relation to the rank of random vectors. Our derivation relies on higher-order moment bounds for quantities involving such random vectors, which may be of independent interest. In Section \ref{sec:nystrom} we first establish an oblivious subspace injection property, and then combine it with the results of Section~\ref{sec:gh} to obtain results on Nystr\"om++ with Gaussian random TT vectors.

\section{Preliminaries on Gaussian random TT vectors}
\label{sec:prelim}

A tensor $X$ of order $d$ is a $d$-dimensional array. For simplicity, we will often assume that all modes have equal size, so that $X \in \R^{n\times n\times \cdots \times n}$.
The vectorization $x = \Vec(X) \in \R^{n^d}$ is obtained by arranging the entries of $X$ in anti-lexicographical order.  We use multi-index notation 
to access the entries of $x$: Given $\mi = (i_1, \ldots, i_d) \in \{1, 2, \ldots, n\}^d$, one writes
$$
    x_\mi = X(i_1,i_2,\ldots,i_d).
$$
We say that $x$ is a \emph{tensor train (TT) vector} if there are
matrices 
$\Xmpow{1} \in \R^{n \times r}$, 
$\Xmpow{d} \in \R^{r \times n}$, and third-order tensors 
$\Xmpow{j} \in \R^{r \times n \times r}$, $2 \leq j \leq d-1$ such that $x$ can be expressed through a contraction
of these so called \emph{TT cores} as follows:
\begin{equation}
 \label{eq:ttvector}
 x_\mi = \sum_{r_1, r_2, \ldots, r_{d-1} = 1}^r \Xmpow{1}_{i_1, r_1} \Xmpow{2}_{r_1, i_2, r_2} \ldots \Xmpow{d}_{r_{d-1}, i_d}
\end{equation}
In practice, the TT ranks (that is, the summation  ranges in the contractions between adjacent TT cores) may vary from one contraction to another. To simplify notation and as there is no major difference from a theoretical point of view, we assume that all TT ranks are equal to a single rank $r$.

Up to scaling, a Gaussian random TT vector is obtained by choosing each TT core independently and Gaussian~\cite{DaasEtAl23,rakhshan2020tensorized}.
\begin{definition}
    \label{def:gaussian-random-tt-vector}
    A Gaussian random TT vector $\omega \in \R^{n^d}$ of rank $r$ takes the form 
    \begin{equation}
        \label{eq:random-tt-definition}        
        \omega_{\mi} = 
            \frac{1}{\sqrt{r^{d-1}}} \sum_{r_1, r_2, \ldots, r_{d-1}} \Ompow{1}_{i_1, r_1} \Ompow{2}_{r_1, i_2, r_2} \ldots \Ompow{d}_{r_{d-1}, i_d},     
            \quad \mi \in \{1, 2, \ldots, n\}^d,
    \end{equation}
    where every entry of the TT cores
    $\Ompow{1}, \ldots, \Ompow{d}$
    is an independent standard normal random variable.
\end{definition}

Gaussian random TT vectors are centered and isotropic.
\begin{lemma}
    \label{lemma:isotropic}
    Given a Gaussian random TT vector $\omega$, it holds that $\EE[\omega] = 0$ and $\EE[\omega \omega^T] = I$.
\end{lemma}
\begin{proof}
This result can be found in, e.g.,~\cite{rakhshan2020tensorized}. We include its proof for completeness.
For each summand in~\eqref{eq:random-tt-definition}, it holds that
$$
    \EE[\Ompow{1}_{i_1, r_1} \Ompow{2}_{r_1, i_2, r_2} \ldots \Ompow{d}_{r_{d-1}, i_d}] = \EE[\Ompow{1}_{i_1, r_1}] \EE[\Ompow{2}_{r_1, i_2, r_2}] \ldots \EE[\Ompow{d}_{r_{d-1}, i_d}] = 0,
$$
using the independence of TT cores. Hence, $\omega$ is centered. To establish isotropy, it again follows from independence that 
    \begin{align*}        
        \EE[\omega_{\mi} \omega_{\mj}]
            &= \frac{1}{r^{d-1}} \sum_{r_1^{\mi}, \ldots, r_{d-1}^{\mi} \atop r_1^{\mj}, \ldots, r_{d-1}^{\mj}} 
                \EE[\Ompow{1}_{i_1, r_1^{\mi}} \Ompow{1}_{j_1, r_1^{\mj}}] \; \cdot \;
                \EE[\Ompow{2}_{r_1^{\mi}, i_2, r_2^{\mi}} \Ompow{2}_{r_1^{\mj}, j_2, r_2^{\mj}}] \; \cdot \; 
                \ldots  \; \cdot \; 
                \EE[\Ompow{d}_{r_{d-1}^{\mi}, i_d} \Ompow{d}_{r_{d-1}^{\mj}, j_d}]
    \end{align*}
    holds for two multiindices $\mi$, $\mj$. It follows that $\EE[\omega_{\mi} \omega_{\mj}] = 0$ when $\mi \neq \mj$. When  $\mi = \mj$, we observe that
    \begin{align*}        
        \EE[\omega_{\mi} \omega_{\mi}]
            &= \frac{1}{r^{d-1}} \sum_{r_1, \ldots, r_{d-1}} 
                \EE[(\Ompow{1}_{i_1, r_1})^2] \; \cdot \;
                \EE[(\Ompow{2}_{r_1, i_2, r_2})^2] \; \cdot \; 
                \ldots  \; \cdot \; 
                \EE[(\Ompow{d}_{r_{d-1}, i_d})^2] = 1, 
    \end{align*}
    where the last equality follows from the well-known fact that the second moment of a standard normal random variable equals $1$.    
\end{proof}

\section{Girard--Hutchinson trace estimation}
\label{sec:gh}

Given a symmetric positive semidefinitive matrix $A \in \R^{n^d \times n^d}$, the classical Girard--Hutchinson trace estimator approximates the trace $\trace(A)$ through sampling quadratic forms,
$$
    H_{\ell}(A) := \frac{1}{\ell} \sum_{i=1}^\ell \omega_i^T A \omega_i
$$
for independent (usually identically distributed) random vectors $\omega_i$. When choosing Gaussian random TT vectors $\omega_i$, it follows immediately from Lemma~\ref{lemma:isotropic} that $H_{\ell}(A)$ is an unbiased trace estimator:
$$
    \EE[H_{\ell}(A)] = \EE[\omega_1^T A \omega_1 ] = \EE[\trace( \omega_1 \omega_1^T  A ) ] = \trace(A).
$$
In order to derive probabilistic guarantees on the estimator's accuracy, we first need to derive some moment bounds.

\subsection{Moment bounds for Gaussian random TT vectors}

The following elementary lemma will be helpful when deriving moment bounds for quantities involving Gaussian random TT vectors.
\begin{lemma}
    \label{lemma:expected-Fnorm-moment}
    Consider a random matrix $\Omega \in \R^{s \times r}$ with independent columns distributed according to $\normal(0, XX^T)$ for some $X \in \R^{s \times t}$. Then
    $$
        \EE[\; \|\Omega\|_F^{2p}\; ] \leq r(r+2)(r+4)\ldots (r+2(p-1)) \cdot \|X\|_F^{2p}, \quad \text{for all integers } p \geq 1.
    $$
    Equality holds when $X$ has rank at most $1$.
\end{lemma}
\begin{proof}
By definition, we can write $\Omega = X \tilde{\Omega}$ for a Gaussian random matrix $\tilde{\Omega} \in \R^{t \times r}$.
Considering the singular value decomposition $X = U \Sigma V^T$ with orthogonal square matrices $U,V$, we let $\hat{\Omega} = V^T \tilde{\Omega} \in \R^{t \times r}$, which is again a Gaussian random matrix. We then have
    $$
        \|\Omega\|_F^2
            = \|U \Sigma V^T \tilde{\Omega}\|_F^2
            = \|\Sigma \hat{\Omega}\|_F^2 = \sigma_1^2 \chi_{r}^{(1)} + \cdots + \sigma_m^2 \chi_{r}^{(m)},
    $$
    where $m=\min\{s, t\}$, and $\chi_{r}^{(i)}$ are independent chi-squared random variables with $r$ degrees of freedom. 
    Using the Minkowski inequality, and a well-known formula for the $p$th moment of a chi-squared variable, we get
    \begin{align*}        
        \EE[\;\|\Omega\|_F^{2p}\;]^{1/p}
            &=\EE\Big[ \Big(\sum_{i=1}^m \sigma_i^2 \chi_{r}^{(i)}\Big)^p \Big]^{1/p}
            \leq \sum_{i=1}^m \EE[ (\sigma_i^2 \chi_{r}^{(i)})^p ]^{1/p}
            = \sum_{i=1}^m \sigma_i^2 \EE[ (\chi_{r}^{(i)})^p ]^{1/p} \\
            &= \left( r(r+2)(r+4)\ldots (r+2(p-1)) \right)^{1/p} \|X\|_F^2.
    \end{align*}
    Note that the inequality above becomes an equality when $\sigma_i = 0$ for $i \ge 2$.
\end{proof}

Before proceeding,  it will be beneficial to introduce some additional notation commonly used in the TT literature. Given a Gaussian random TT vector $\omega$ defined in~\eqref{eq:random-tt-definition} and 
$1 \le k \le d-1$, we let $\Omega^{\leq k} \in \R^{n^k \times r}$ denote the \emph{left interface matrix} obtained by contracting the first $k$ TT cores:
$$
    \Omega^{\leq k}_{(i_1,\ldots,i_k),\,r_k}
        = \Omega^{\leq k}_{\milek, r_k} = 
            \sum_{r_1,\ldots,r_{k-1}}
            \Omega^{[1]}_{i_1,r_1}
            \Omega^{[2]}_{r_1,i_2,r_2}
            \cdots
            \Omega^{[k]}_{r_{k-1},i_k,r_k},
$$
with the multiindex $\milek=(i_1, \ldots, i_k)$. The definition of a TT vector implies nestedness relations among the interface matrices:
\begin{equation}
    \label{eq:tt}
       \Omega^{\leq k}_{\milek, r_k} = \sum_{r_{k-1}} \Omega^{\leq k-1}_{\mi_{\le k-1}, r_{k-1}} \Omega^{[k]}_{r_{k-1},i_k, r_k} \text{ for } 2 \leq k \le d-1, \quad 
       \omega_{\mi} =  \frac{1}{\sqrt{r^{d-1}}} \sum_{r_{d-1}} \Omega^{\leq d-1}_{\mile{d-1}, r_{d-1}} \Omega^{[d]}_{r_{d-1},i_d}.
\end{equation}

An arbitrary vector $u \in \R^{n^d}$ can be reshaped / unfolded into matrices in many different ways. Particularly relevant to TT, the TT unfoldings
$U^{(k)} \in \R^{n^k\times n^{d-k}}$, for $k = 1,\ldots,d-1$, are defined entry-wise by
$$
    U^{(k)}_{\milek, \migk} = u_{\mi},\quad 
        \milek=(i_1, \ldots, i_k), \quad \migk=(i_{k+1}, \ldots, i_d).
$$

\begin{lemma}
    \label{lemma:pth-moment-tower}
    Considering a Gaussian random TT vector $\omega \in \R^{n^d}$ of rank $r$
    and a vector $u \in \R^{n^d}$ such that $\|u\|_2 = 1$, define $\plek := (U^{(k)})^T \Omlek \in \R^{n^{d-k} \times r}$ using the notation defined above. Then the following statements hold for every integer $p \geq 1$:
    \begin{itemize}
        \item[(a)] $\EE[ (u^T \omega)^{2p} \; | \; \Ompow{1}, \Ompow{2}, \ldots, \Ompow{d-1} ] = \frac{(2p-1)!!}{r^{p(d-1)}} \|\ple{d-1}\|_F^{2p}$,
        \item[(b)] $\EE[ \|\ple{k}\|_F^{2p} \; | \; \Ompow{1}, \Ompow{2}, \ldots, \Ompow{k-1} ] \leq r^{p}(1+\frac{2}{r})(1+\frac{4}{r}) \cdots (1+\frac{2(p-1)}{r}) \|\ple{k-1}\|_F^{2p}$, \\ for $k=d-1, d-2, \ldots, 3, 2$,
        \item[(c)] $\EE[ \|\ple{1}\|_F^{2p}] \leq r^{p}(1+\frac{2}{r})(1+\frac{4}{r}) \cdots (1+\frac{2(p-1)}{r})$.
    \end{itemize}
    If $u = u_1 \kron u_2 \kron \ldots \kron u_d$ for vectors $u_i \in \R^n$, then equality holds in all statements.
\end{lemma}
\begin{proof}
\begin{itemize}
\item[(a)]
    Using \eqref{eq:tt}, we find that
    \begin{align*}
        u^T \omega
            &= \sum_{\mi} u_\mi \omega_\mi 
            = \frac{1}{\sqrt{r^{d-1}}}\sum_{\mile{d-1}, i_{d}, r_{d-1}} u_{\mile{d-1}, i_d} \Omle{d-1}_{\mile{d-1}, r_{d-1}} \Ompow{d}_{r_{d-1},i_d} 
            = \frac{1}{\sqrt{r^{d-1}}} \trace\big( \ple{d-1} \Ompow{d}\big).
    \end{align*}
    Conditioned on $\Ompow{1}, \ldots, \Ompow{d-1}$, the matrix $\ple{d-1}$ becomes deterministic and
    $$
        \trace\big( \ple{d-1} \Ompow{d}\big) \sim \normal(0, \|\ple{d-1}\|^2_F),
    $$
    using that the left expression can be seen as the inner product between the vectorization of $\ple{d-1}$ and a Gaussian random vector.
    Since $\EE[\gamma^{2p}] = (2p-1)!!\sigma^{2p}$ for $\gamma \sim \normal(0, \sigma^2)$,
    it follows that
    $$            
        \EE\big[ (u^T \omega)^{2p} \; | \; \Ompow{1}, \ldots, \Ompow{d-1} \big] 
            = \frac{(2p-1)!!}{r^{p(d-1)}}\|\ple{d-1}\|^{2p}_F.
    $$
    
\item[(b)]
    Consider the entry $(\migk, r_k)$ of the matrix $\plek \in \R^{n^{d-k} \times r}$, and use~\eqref{eq:tt}:
    \begin{align}
        \plek_{\migk, r_k} 
            &= \sum_{\milek} U^{(k)}_{\milek,\migk} \Omle{k}_{\milek, r_k}  
            = \sum_{\mi_{\le k-1}, i_k} \sum_{r_{k-1}} U^{(k)}_{\mi_{\le k-1}, i_k, \migk} \Omle{k-1}_{\mi_{\le k-1},r_{k-1}} \Ompow{k}_{r_{k-1},i_k,r_k}  \nonumber  \\ 
            &= \sum_{i_k, r_{k-1}} \ple{k-1}_{(i_k, \migk), r_{k-1}} \Ompow{k}_{r_{k-1},i_k,r_k}. \label{eq:blabla}
    \end{align}   
    Conditioned on $\Ompow{1}, \ldots, \Ompow{k-1}$, the matrix $\ple{k-1}$ becomes deterministic.
    Reshaping $\ple{k-1}$ into the matrix $\tildeple{k-1} \in \R^{n^{d-k} \times rn}$ by setting
$\widetilde{\Pi}^{\leq k-1}_{I_{>k},\,(i_k,r_{k-1})}:=
\Pi^{\leq k-1}_{(i_k,I_{>k}),\,r_{k-1}}$, it follows from~\eqref{eq:blabla} that the $r_k$-th column
    $\plek_{:, r_k}$ of $\plek$ is given by the product of $\tildeple{k-1}$ with the Gaussian random vector obtained from the vectorization of the slice
    $\Ompow{k}_{:,:,r_k}$. It follows that the columns $\plek_{:, r_k}$ are independent and
    $$
        \plek_{:, r_k} \sim \normal(0, \tildeple{k-1}(\tildeple{k-1})^T).
    $$
    Applying Lemma \ref{lemma:expected-Fnorm-moment} gives
    \begin{align*}        
        \EE[\|\plek\|_F^{2p} \; | \; \Ompow{1}, \Ompow{2}, \ldots, \Ompow{k-1}]
            &\leq r^{p}\Big(1+\frac{2}{r}\Big)\Big(1+\frac{4}{r}\Big)\cdots\Big(1+\frac{2(p-1)}{r}\Big) \|\tildeple{k-1}\|_F^{2p} \\
            &= r^{p}\Big(1+\frac{2}{r}\Big)\Big(1+\frac{4}{r}\Big)\cdots\Big(1+\frac{2(p-1)}{r}\Big) \|\ple{k-1}\|_F^{2p}.
    \end{align*}

\item[(c)]
    Consider the entry $(\mig{1}, r_1)$ of the matrix $\ple{1} \in \R^{n^{d-1} \times r}$ and note that $\Omle{1} = \Ompow{1}$:
    \begin{align*}
        \ple{1}_{\mig{1}, r_1} 
            &= \sum_{\mile{1}} U^{(1)}_{\mile{1},\mig{1}} \Omle{1}_{\mile{1}, r_1}  
            = \sum_{i_1} U^{(1)}_{i_1,\mig{1}} \Ompow{1}_{i_1, r_1}.
    \end{align*}
    As in case (b), it follows that the columns $\ple{1}_{:, r_1}$ are independently distributed according to
    $$
        \ple{1}_{:, r_1} \sim \normal(0, (U^{(1)})^TU^{(1)}).
    $$
Thus, Lemma \ref{lemma:expected-Fnorm-moment} applies again and gives
    \begin{align*}
        \EE[\|\ple{1}\|_F^{2p}] 
            &\leq r^{p} \Big(1+\frac{2}{r}\Big)\Big(1+\frac{4}{r}\Big)\ldots\Big(1+\frac{2(p-1)}{r}\Big) \|U^{(1)}\|_F^{2p}.
    \end{align*}
    The claim follows by noticing that $\|U^{(1)}\|_F = \|u\|_2 = 1$.
\end{itemize}
If $u=u_1 \kron u_2 \kron \ldots \kron u_d$, it follows that all TT unfoldings $U^{(k-1)} = (u_1 \kron \ldots \kron u_{k-1})(u_{k} \kron \ldots \kron u_d)^T$ are of rank 1, and so are $\ple{k-1} = (U^{(k-1)})^T \Omle{k-1} = (u_{k} \kron \ldots \kron u_d) b_{k-1}^T$, for some vector $b_{k-1}$. Then $\tildeple{k-1} = (u_{k+1} \kron \ldots \kron u_d) (b_{k-1} \kron u_k)^T$, which is once again a matrix of rank 1.
Thus, using the last statement of Lemma \ref{lemma:expected-Fnorm-moment}, all inequalities appearing in the proof of (b) and (c) turn into equalities.
\end{proof}

Recursively applying Lemma \ref{lemma:pth-moment-tower} gives the main result of this section.

\begin{theorem}
    \label{tm:moment}  
        Consider a Gaussian random TT vector $\omega \in \R^{n^d}$ of rank $r$
    and a vector $u \in \R^{n^d}$ such that $\|u\|_2 = 1$. 
    Then, for all integers $p \geq 1$, it holds that 
    $$
        \EE[(u^T \omega)^{2p}] 
            \leq (2p-1)!! \Big(\Big(1 + \frac{2}{r}\Big)\Big(1 + \frac{4}{r}\Big)\cdots\Big(1 + \frac{2(p-1)}{r}\Big) \Big)^{d-1}
            \leq (2p-1)!! \; e^{\frac{p^2(d-1)}{r}}.
    $$
    If $u = u_1 \kron u_2 \kron \ldots \kron u_d$ for vectors $u_i \in \R^n$, then the first inequality above turns into an equality.
\end{theorem}
\begin{proof}
    The first inequality is obtained by repeatedly combining the tower rule for expectation with Lemma~\ref{lemma:pth-moment-tower}. The second inequality follows from noting
    that $$
        \log\prod_{i=1}^{p-1} \Big(1 + \frac{2i}{r}\Big)
            = \sum_{i=1}^{p-1} \log\Big(1 + \frac{2i}{r}\Big)
            \leq \sum_{i=1}^{p-1} \frac{2i}{r}
            = \frac{p(p-1)}{r} 
            \leq \frac{p^2}{r}.
    $$
\end{proof}

We will use the following variation of standard results for subexponential random variables~\cite{Vershynin2018High-dimensional}.
\begin{proposition} \label{prop:Lp-mean-bound}
Let $p\geq 2$ be an even integer, and let
$X_1,\ldots,X_\ell$ be independent centered random variables satisfying
$$
    \|X_i\|_{L_q}\leq Lq,
        \qquad i=1,\ldots,\ell,
        \qquad q=2,\ldots,p.
$$
Then
$$
    \Big\|
        \frac{1}{\ell}\sum_{i=1}^{\ell}X_i
    \Big\|_{L_p}
    \leq e^{3/2}L
    \left(
        \sqrt{\frac{p}{\ell}}
        +
        \frac{p}{\ell}
    \right).
$$
\end{proposition}

\begin{proof}
We first note that, for every integer $q\geq 2$,
$$
    q^q\leq 2q!e^{q-2}.
$$
Indeed, the sequence
$
    a_q:=\frac{2q!e^{q-2}}{q^q}
$
satisfies $a_2=1$ and
$
    \frac{a_{q+1}}{a_q}
    =
    e\big(\frac{q}{q+1}\big)^q
    =
    \frac{e}{(1+1/q)^q}
    >1.
$
Consequently,
\begin{equation}
    \mathbb{E}|X_i|^q
        \leq (Lq)^q
        \leq 2q!e^{q-2}L^q,
        \qquad q=2,\ldots,p. \label{eq:aux1}
\end{equation}

Set
$
    S:=X_1 + \cdots + X_\ell.
$
Because $p$ is even, $|S|^p=S^p$. The multinomial theorem and independence therefore give
$$ 
    \mathbb{E}|S|^p 
        = \mathbb{E}S^p 
        = \sum_{\alpha_1+\cdots+\alpha_\ell=p} \frac{p!}{\alpha_1!\cdots\alpha_\ell!} \prod_{i=1}^{\ell}\mathbb{E}[X_i^{\alpha_i}]. 
$$
Since $\mathbb{E}X_i=0$, every summand for which $\alpha_i=1$ for at least one index $i$ vanishes. Hence,
\begin{equation} \label{eq:aux22}
    \mathbb{E}|S|^p 
        = \sum_{\substack{\alpha_1+\cdots+\alpha_\ell=p\\ \alpha_i\neq 1\ \text{for all }i}} \frac{p!}{\alpha_1!\cdots\alpha_\ell!} \prod_{i=1}^{\ell}\mathbb{E}[X_i^{\alpha_i}]
        \leq p! \sum_{\substack{\alpha_1+\cdots+\alpha_\ell=p\\ \alpha_i\neq 1\ \text{for all }i}} \prod_{i=1}^{\ell} \frac{\mathbb{E}|X_i|^{\alpha_i}}{\alpha_i!}.  
\end{equation}
For each $i$, introduce the polynomial 
$$
    F_i(z) := 1+\sum_{q=2}^{p} \frac{\mathbb{E}|X_i|^q}{q!}z^q.
$$
Then the Cauchy product formula yields
\begin{equation}\label{eq:cauchy}
    \mathbb{E}|S|^p
        \leq p!\,[z^p]\prod_{i=1}^{\ell}F_i(z),
\end{equation}
where $[z^p]$ denotes the coefficient of $z^p$.

Let $\preccurlyeq$ denote coefficientwise comparison of power series.
By~\eqref{eq:aux1},
$$
    F_i(z)
        \preccurlyeq 1+\frac{2L^2z^2}{1-eLz}
        \preccurlyeq \exp\left( \frac{2L^2z^2}{1-eLz} \right).
$$
Hence
$$
    \prod_{i=1}^{\ell}F_i(z)
        \preccurlyeq \exp\left( \frac{2\ell L^2z^2}{1-eLz} \right).
$$
Since the latter power series has nonnegative coefficients, for every
$0<t<(eL)^{-1}$,
$$
    [z^p]\prod_{i=1}^{\ell}F_i(z)
    \leq t^{-p} \exp\left( \frac{2\ell L^2t^2}{1-eLt} \right).
$$
Together with~\eqref{eq:cauchy}, this gives
\begin{equation}
    \label{eq:aux3}
    \frac{\mathbb{E}|S|^p}{p!}
        \leq t^{-p} \exp\left( \frac{2\ell L^2t^2}{1-eLt} \right).
\end{equation}
Now, choose
$
    t := L^{-1}\bigl(e+2\sqrt{\ell/p}\bigr)^{-1}. 
$
Then
$$
    \frac{2\ell L^2t^2}{1-eLt}
        = \frac{\sqrt{\ell p}}{e+2\sqrt{\ell/p}}
        \leq \frac{p}{2}.
$$
Taking $p$th roots in~\eqref{eq:aux3} yields
$$
    \|S\|_{L_p}
    \leq (p!)^{1/p}e^{1/2}t^{-1} \\
    \leq e^{3/2}Lp + 2\sqrt{e}\,L\sqrt{\ell p}.
$$
Dividing by $\ell$ and using $2\sqrt{e}\leq e^{3/2}$ completes the proof.
\end{proof}

\subsection{Tail bounds for Girard--Hutchinson trace estimation}

We are now ready to state a tail bound for the Girard--Hutchinson trace estimation. 
Provided that the rank of the Gaussian random TT vectors is at least $d-1$, the number of samples required to achieve a fixed relative accuracy and failure probability is independent of the tensor order $d$.

\begin{theorem}
    \label{tm:hutch-4th-moment}
    For a symmetric positive semidefinite matrix $A \in \R^{n^d \times n^d}$, consider the Girard--Hutchinson trace estimator $H_\ell(A)$ defined in~\eqref{eq:GH} using $\ell$ independent random Gaussian TT vectors $\omega_i \in \R^{n^d}$ of rank $r \geq d-1$. For given $\eps > 0$ and $0 < \delta < 1$, it holds that 
    $$
        \PP{|H_\ell(A) - \trace(A)| > \eps \, \trace(A)} \leq \delta,
    $$
    provided that 
    $
        \ell \geq 22 \, \eps^{-2} \delta^{-1}.
    $
\end{theorem}
\begin{proof}
    Let $A = B B^T$ for a square matrix $B$, and let $\omega \sim \omega_i$.  Then
    \begin{align*}    
        \Var[\omega^T A \omega]
            &= \EE[ (\omega^T A \omega - \trace(A))^2] \\
            &= \EE[ (\|B^T \omega\|^2 - \|B\|_F^2)^2] \\
            &= \EE[ \|B^T \omega\|^4 ] - 2 \|B\|_F^2 \cdot \EE[\|B^T \omega\|^2] + \|B\|_F^4.
    \end{align*}
    Let $B = [b_1 \; b_2 \; \ldots \; b_N]$, where $b_i$ are the columns of $B$, and $N=n^d$. Since $\omega$ is isotropic,
    \begin{align*}
        \EE[\|B^T \omega\|^2]
            &= \EE\Big[ \sum_{i=1}^N (b_i^T \omega)^2 \Big] 
            = \sum_{i=1}^N \EE[ (b_i^T \omega)^2 ] 
            = \sum_{i=1}^N \EE[ b_i^T \omega \omega^T b_i ] 
            = \sum_{i=1}^N b_i^T \EE[\omega \omega^T] b_i 
            = \sum_{i=1}^N \|b_i\|^2 
            = \|B\|_F^2.
    \end{align*} 
    For the fourth moment, we first apply the Cauchy--Schwarz inequality for expectation, and then use Theorem \ref{tm:moment} with $p=2$ for each (normalized) column of $B$:
    \begin{align*}
        \EE[ \|B^T \omega\|^4 ]
            &= \EE\Big[ ( \sum_{i=1}^N (b_i^T \omega)^2 )^2 \Big] 
            = \sum_{i, j=1}^N \EE[ (b_i^T \omega)^2 (b_j^T \omega)^2 ] 
            \leq \sum_{i, j=1}^N \EE[ (b_i^T \omega)^4]^{1/2} \EE[(b_j^T \omega)^4 ]^{1/2} \\
            &\leq 3 (1 + \frac{2}{r})^{d-1} \sum_{i, j=1}^N \|b_i\|^2 \|b_j\|^2 
            = 3 (1 + \frac{2}{r})^{d-1} \Big(\sum_{i=1}^N \|b_i\|^2\Big)^2 
            = 3 (1 + \frac{2}{r})^{d-1} \|B\|_F^4.
    \end{align*}
    Therefore, using $\|B\|_F^2 = \trace(BB^T) = \trace(A)$,
    \begin{align*}    
        \Var[\omega^T A \omega]
            &= \EE[ \|B^T \omega\|^4 ] - 2 \|B\|_F^2 \cdot \EE[\|B^T \omega\|^2] + \|B\|_F^4 \\
            &\leq \Big(3 \Big(1 + \frac{2}{r}\Big)^{d-1} - 1 \Big) \|B\|_F^4 
            = \Big(3 \Big(1 + \frac{2}{r}\Big)^{d-1} - 1 \Big) (\trace(A))^2.
    \end{align*}
    With $r \geq d-1$, we have $(1 + \frac{2}{r})^{d-1} \leq e^2$, and
    $
        \Var[\omega^T A \omega]
            \leq (3 e^2 - 1 ) (\trace(A))^2 \leq 22 (\trace(A))^2
    $.
    Since $\Var[H_\ell] = \frac{1}{\ell} \Var[\omega^T A \omega]$, we apply the Chebyshev inequality to arrive at
    \begin{align*}
        \PP{|H_\ell - \trace(A)| > \eps \, \trace(A)}
            &\leq \frac{\Var[H_\ell]}{(\eps \, \trace(A))^2}
            \leq 22 \, \ell^{-1} \eps^{-2} \leq \delta
    \end{align*}
    for $\ell \geq 22 \, \eps^{-2} \delta^{-1}$.
\end{proof}

The order-independent sample complexity obtained in Theorem~\ref{tm:hutch-4th-moment} for Gaussian random TT vectors of rank at least $d-1$ stands in stark contrast to that of Khatri--Rao random vectors, equivalently Gaussian random TT vectors of rank $1$, for which the worst-case sample complexity grows exponentially with $d$~\cite{AvronMeyer25,CamanoEpperlyMeyerTropp25}.

The $\redvel(\eps^{-2} \delta^{-1})$ sample complexity established by Theorem~\ref{tm:hutch-4th-moment} has a worse dependence on the failure probability than the $\redvel(\eps^{-2} \log (1/\delta))$ complexity obtained with unstructured Gaussian random vectors (see, e.g., \cite{CortinovisKressner22}). The standard median-of-means amplification technique recovers the logarithmic dependence on the failure probability.
Assume that we compute $k$ independent Girard--Hutchinson estimators $H_\ell^{(1)}, \ldots, H_\ell^{(k)}$, where each one uses $\ell = \lceil 88 \, \eps^{-2} \rceil$ samples---sufficient to have
$$
    \PP{|H_\ell^{(i)} - \trace(A)| > \eps \, \trace(A)} =: q \leq \frac{1}{4}, \quad \text{for all } i=1, \ldots, k.
$$
Let $H_{k, \ell}$ denote the median of $H_\ell^{(1)}, \ldots, H_\ell^{(k)}$. Then the event $|H_{k, \ell} - \trace(A)| > \eps \, \trace(A)$ happens only if at least $k/2$ out of $k$ estimators $H_\ell^{(i)}$ satisfy $|H_\ell^{(i)} - \trace(A)| > \eps \, \trace(A)$. In other words,
$$
    \PP{|H_{k, \ell} - \trace(A)| > \eps \, \trace(A)}
        \leq \PP{S_{k, \ell} \geq k/2},
$$
where $S_{k, \ell} = \sum_{i=1}^k S_\ell^{(i)}$, and $S_\ell^{(i)}$ is the indicator of the event $|H_\ell^{(i)} - \trace(A)| > \eps \, \trace(A)$. Since $S_\ell^{(i)} \sim \text{Bernoulli}(q)$, one can estimate the tail probability of $S_{k, \ell}$ using Chernoff:
\begin{align*}    
    \PP{S_{k, \ell} \geq k/2}
        &\leq \inf_{t > 0} \EE[e^{t S_{k, \ell}}] e^{-tk/2}
        = \inf_{t > 0} \prod_{i=1}^k \EE[e^{t S_{\ell}^{(i)}}] e^{-tk/2}
        = \inf_{t > 0} (1-q+q e^t)^k e^{-tk/2} \\
        &= \inf_{t > 0} \left( \frac{(1-q+q e^t)^2}{e^t}\right)^{k/2}.
\end{align*}
Choosing $t = \log\frac{1-q}{q}$ and $q \leq 1/4$, we arrive at
$$
    \PP{|H_{k, \ell} - \trace(A)| > \eps \, \trace(A)}
        \leq \PP{S_{k, \ell} \geq k/2}
        \leq (4q(1-q))^{k/2}
        \leq (3/4)^{k/2} 
        \leq \delta,
$$
for $k \geq \frac{2}{\log(4/3)} \log(1/\delta)$, i.e., $k \geq 7 \log(1/\delta)$ suffices. Combined with the number of samples needed for each $H_\ell^{(i)}$, we arrive at the following result.

\begin{corollary}
    \label{tm:median-of-means}
    Considering the setting of Theorem~\ref{tm:hutch-4th-moment}, 
    let $H_{k, \ell}$ denote the median of $k \geq 7 \log(1/\delta)$ independent Girard--Hutchinson trace estimates for the trace of $A$, each using $\ell \geq 88 \, \eps^{-2}$ independent Gaussian random TT vectors of rank $r \geq d-1$. Then
    \begin{equation}
        \label{eq:median-of-means-eps-delta}    
        \PP{|H_{k, \ell} - \trace(A)| > \eps \, \trace(A)} \leq \delta
    \end{equation}
    In particular, this shows that a total of $\redvel( \eps^{-2} \log(1/\delta))$ samples of Gaussian random TT vectors of rank $r \geq d-1$ suffices to ensure \eqref{eq:median-of-means-eps-delta}.
\end{corollary}

We next show that by increasing the rank to $r = \redvel(d \log(1/\delta))$, one can ensure~\eqref{eq:median-of-means-eps-delta} with $\redvel(\eps^{-2} \log(1/\delta))$ samples without resorting to the median-of-means approach, i.e., by using only the standard Girard--Hutchinson estimator.

\begin{theorem}
    \label{tm:tt-hutch-as-gaussian}
        Consider the setting of Theorem~\ref{tm:hutch-4th-moment}, and assume that the rank of the random Gaussian TT vectors is set to $r=d p_\ast$ with
        $p_\ast = 2 \lceil \frac12\log  \frac{1}{\delta}\rceil$. Given $0<\eps \le 1$ and $0 < \delta < 1$, it holds that
    $$
        \PP{|H_\ell(A) - \trace(A)| > \eps\, \trace(A)} \leq \delta,
    $$
    provided that $\ell \geq 16 e^7 \eps^{-2} p_\ast$.
\end{theorem}
\begin{proof}
    Let $N = n^d$, let $\omega$ be a Gaussian random TT vector of rank $r$, and consider a spectral decomposition $A = \sum_{i=1}^N \lam_i u_i u_i^T$. Then
    $$
        \omega^T A \omega 
            = \trace(A) \sum_{i=1}^N \frac{\lam_i}{\trace(A)} (u_i^T \omega)^2.
    $$
    Since $t \to t^p$ is convex for $p \geq 1$, applying Jensen's inequality and Theorem \ref{tm:moment} gives
    \begin{equation}
        \label{eq:tm-tt-hutch-bound-pth-moment}
        \EE[(\omega^T A \omega)^p]
            \leq \trace(A)^p \sum_{i=1}^N \frac{\lam_i}{\trace(A)} \EE[(u_i^T \omega)^{2p}]
            \leq \trace(A)^p (2p-1)!! e^{\frac{p^2(d-1)}{r}}.
    \end{equation}
    Setting $p_\ast = 2 \lceil  \frac12 \log \frac{1}{\delta} \rceil$ and $r = dp_\ast$, it holds that $\frac{p^2(d-1)}{r} \leq p$ for all $1 \leq p \leq p_\ast$ and 
    $
        \|\omega^T A \omega\|_{L_p} \leq e p \cdot \trace(A)
    $
    since $(2p-1)!! \leq p^p$.
    After centering, we have
    $$
        \|\omega^T A \omega - \trace(A)\|_{L_p} \leq 2e p \cdot \trace(A).
    $$
    Therefore, for iid instances $\omega_i \sim \omega$, Proposition \ref{prop:Lp-mean-bound} gives 
    \begin{align*}
        \|H_\ell - \trace(A)\|_{L_{p_\ast}} 
            = \Big\|\frac{1}{\ell} \sum_{i=1}^\ell (\omega_i^T A \omega_i - \trace(A))\Big\|_{L_p} 
            \leq 2e^{5/2} \cdot \trace(A) \left( \sqrt{\frac{p_\ast}{\ell}} + \frac{p_\ast}{\ell} \right) \leq 4e^{5/2} \cdot \trace(A) \sqrt{\frac{p_\ast}{\ell}},
    \end{align*}
    where we used $\ell \geq p_\ast$.
    Using Markov's inequality,
    \begin{align*}
        \PP{|H_\ell - \trace(A)| > \eps\, \trace(A)}
            \leq \left( \frac{4e^{5/2} \cdot \trace(A) \sqrt{\frac{p_\ast}{\ell}}}{\eps \trace(A)} \right)^{p_\ast}
    \end{align*}
    which is at most  $\delta$ when
    $$
        \ell 
            \geq 16 e^{5} \cdot \eps^{-2} \cdot p_\ast \cdot \delta^{-2/p_\ast}.
    $$
    Since $\delta^{-2/p_\ast} \leq e^2$, the theorem statement holds with $\ell \geq 16 e^7 \eps^{-2} p_\ast$.
\end{proof}

\begin{remark}
When $r=\mathcal O(d)$, the preceding moment argument no longer yields the logarithmic dependence on $\delta^{-1}$ obtained in Theorem~\ref{tm:tt-hutch-as-gaussian}. Assuming $r = \lceil \alpha^{-1} d \rceil$ for some constant $0 < \alpha \le 1$, by using equation \eqref{eq:tm-tt-hutch-bound-pth-moment} we arrive at
    \begin{equation}
        \label{eq:norm-log-normal}
        \|\omega^T A \omega\|_{L_p} \leq p e^{\alpha p} \, \trace(A).
    \end{equation}
    This bound exhibits log-normal-type growth in the moment order $p$ and is indicative of a heavier-than-exponential tail.

For a rank-one matrix $A=uu^T$ with a unit-norm vector $u = u_1 \kron u_2 \kron \ldots \kron u_d$,
Theorem~\eqref{tm:moment} gives an exact moment formula showing that the upper bound~\eqref{eq:norm-log-normal} has the correct qualitative moderate-moment growth.
More specifically, we have 
    $$
        \EE[(\omega^T A \omega)^{p}] 
            = \EE[(u^T \omega)^{2p}]
            = (2p-1)!! \left((1 + \frac{2}{r})(1 + \frac{4}{r})\ldots(1 + \frac{2(p-1)}{r}) \right)^{d-1}.
    $$
    Note that $(p/e)^p \leq (2p-1)!! \leq p^p$, and, since $\frac{x}{2} \leq \log(1+x) \leq x$ for $x \in [0, 1]$, we have
    $$
        \log\prod_{i=1}^{p-1} (1 + \frac{2i}{r})
            = \sum_{i=1}^{p-1} \log(1 + \frac{2i}{r})
            \geq \sum_{i=1}^{p-1} \frac{i}{r}
            = \frac{p(p-1)}{2r}
            \geq \frac{p^2}{8r},
            \quad \text{for } 2 \leq p \leq \frac{r}{2}.
    $$
    Therefore, assuming for simplicity that $r = \alpha^{-1} d$ is an integer, and assuming $d\ge 2$,
    $$
        \frac{p}{e} e^{\alpha p/16} 
        \leq \frac{p}{e}e^{\frac{p(d-1)}{8r}} \leq \|\omega^T A \omega\|_{L_p} \leq p e^{\frac{p(d-1)}{r}} \leq p e^{\alpha p}.
    $$
    This shows that, for moderate moments $2\le p\le r/2$, $\|\omega^T A \omega\|_{L_p}$
has log-normal-type growth up to polynomial factors.
    \end{remark}

\section{Oblivious subspace injections and the Nystr\"{o}m++ algorithm}
\label{sec:nystrom}

The number of samples required for the Girard--Hutchinson trace estimator scales with $\eps^{-2}$, which is not satisfactory when the required relative accuracy is high. Using a variance reduction technique such as in the Hutch++ algorithm \cite{MeyerMuscoMusco21} improves the sample complexity to $\redvel(\eps^{-1})$. When using Gaussian random matrices $\Omega$, $\Psi$ in Algorithm \ref{alg:hutchPP}, one can choose $\ell_1, \ell_2 = \redvel(\eps^{-1} \log(1/\delta))$ to ensure $\PP{|\trace^{\text{H++}}_{\ell_1, \ell_2} - \trace(A)| > \eps \trace(A)} < \delta$. When the columns of the matrices $\Omega$, $\Psi$ are composed of random TT vectors of low rank, the orthogonalization step needed by the randomized SVD in Hutch++ may lead to significant rank increase and, in turn, potentially increasing the computational cost dramatically. 
\begin{algorithm}
    \caption{Hutch++ \cite[Algorithm 1]{MeyerMuscoMusco21}}
    \label{alg:hutchPP}
    \begin{algorithmic}
        \State \textbf{Input:} Symmetric positive semidefinite matrix $A \in \R^{N \times N}$; number of samples $\ell_1, \ell_2$.
        \State \textbf{Output:} Approximation of $\trace(A)$.
        \State Sample independent random matrices $\Omega \in \R^{N \times \ell_1}$ and $\Psi \in \R^{N \times \ell_2}$.
        \State Compute an orthonormal basis $Q$ for the span of $A\Omega$.
        \State \textbf{Return} $\trace^{\text{H++}}_{\ell_1, \ell_2}(A) := \trace(Q^TAQ) + \frac{1}{\ell_2}\trace(\Psi^T(I-QQ^T)A(I-QQ^T)\Psi)$.
    \end{algorithmic}
\end{algorithm}

\begin{algorithm}
    \caption{Nystr\"{o}m++ \cite[Algorithm 5]{PerssonCortinovisKressner22}}
    \label{alg:nystromPP}
    \begin{algorithmic}
        \State \textbf{Input:} Symmetric positive semidefinite matrix $A \in \R^{N \times N}$; number of samples $\ell_1$, $\ell_2$.
        \State \textbf{Output:} Approximation of $\trace(A)$.
        \State Sample random matrices $\Omega \in \R^{N \times \ell_1}$ and $\Psi \in \R^{N \times \ell_2}$.
        \State Compute $[X,\, Y] = A[\Omega,\, \Psi]$.
        \State \textbf{Return} $\trace^{\text{N++}}_{\ell_1, \ell_2}(A) := \trace((\Omega^T X)^\pinv (X^T X)) + \frac{1}{\ell_2}( \trace(\Psi^T Y) - \trace(\Psi^T X (\Omega^T X)^\pinv (X^T \Psi))$.
    \end{algorithmic}
\end{algorithm}

When replacing the randomized SVD by the Nystr\"{o}m approximation, this orthogonalization step is avoided, leading to variants of the Hutch++ algorithm such as NA-Hutch++ \cite{MeyerMuscoMusco21} and Nystr\"{o}m++ \cite{PerssonCortinovisKressner22}.
This allows for efficient computation in combination with random TT vectors, similarly as with random Khatri--Rao vectors \cite[Section 6.2]{CamanoEpperlyMeyerTropp25}.
In particular, Algorithm~\ref{alg:nystromPP} can be implemented efficiently when $A$ has a structure compatible with TT vectors, such as a Kronecker product or matrix product operator structure.
Under some mild spectral decay, we will prove that the sample complexity of random TT vectors in this task does not depend on the tensor order, requiring only $\ell_1, \ell_2 = \redvel(\eps^{-1})$ samples in Algorithm \ref{alg:nystromPP}. This is unlike the Khatri--Rao case, where the sample size may need to be exponential in tensor order to reach the same accuracy.

\paragraph{Oblivious subspace injections and random TT vectors.}

Our analysis of Nystr\"{o}m++ for random TT vectors will rely on the 
oblivious subspace injection (OSI) property introduced in \cite{CamanoEpperlyMeyerTropp25,Tropp25}.

\begin{definition}
    Let $0 < \eps, \delta < 1$, and $k \in \mathbb{N}$. A random matrix $\Omega \in \R^{N \times \ell}$ is a $(1-\eps, \delta, k)$-oblivious subspace injection (OSI) if
$\EE\big[\Omega \Omega^T] = I_N$, and the property
    \begin{equation}
        \label{eq:osi}
        (1-\eps) \|u\|^2 \leq \|\Omega^T u\|_2^2, \quad \forall u \in \mathcal{U},
    \end{equation}
    holds for every fixed $k$-dimensional subspace $\mathcal{U}$ of $\R^N$ with probability at least $1-\delta$.
\end{definition}

The following theorem is essential in proving that a random matrix is an OSI.

\begin{theorem}[{\cite[Theorem 5.2]{Tropp25}}]
    \label{tm:tropp-mineig}
    Let $\tilde{\omega} \in \R^k$ be a centered random vector with identity covariance matrix such that
    $$
        \EE[(\tilde{u}^T \tilde{\omega})^4] \leq L, \quad \text{ for all } \tilde{u} \in \R^k \text{ such that } \|\tilde{u}\|=1.
    $$
    Consider $Y = \frac{1}{\ell} \sum_{i=1}^{\ell} \tilde{\omega}_i \tilde{\omega}_i^T$ with independent copies $\tilde{\omega}_i$ of $\tilde{\omega}$ and $0 < \eps,\delta < 1$.
    Then
    $
        \ell \geq 24 L \eps^{-2} \max\{k, \log(2k/\delta)\}
    $
    samples suffice to ensure for the smallest eigenvalue $\lam_{\min}(Y)$ of $Y$ that
    $$
        \PP{\lam_{\min}(Y) < 1-\eps} \leq \delta.
    $$
\end{theorem}

We are now ready to show the main result of this section: A random matrix composed of Gaussian random TT vectors of rank $r \geq d-1$ is an OSI. 
\begin{theorem}
    \label{tm:osi}
    Consider a random matrix $\Omega \in \R^{n^d \times \ell}$ defined as
    \begin{equation}
        \label{eq:Omega-with-TT-rows}    
        \Omega = \frac{1}{\sqrt{\ell}} \mb{cccc} \omega_{1} & \omega_{2} & \ldots & \omega_{\ell} \me,
    \end{equation}
    where $\omega_{j}$ are independent Gaussian random TT vectors of rank $r$. Suppose that
    $$
        \ell \geq 72 \cdot \Big(1+\frac{2}{r}\Big)^{d-1} \eps^{-2}\max\{k, \log(2k/\delta)\}.
    $$
    Then $\Omega$ is a $(1-\eps, \delta, k)$-OSI.
    In particular, when choosing $r \geq d-1$ and $\ell \geq 72 e^2 \cdot \eps^{-2} \max\{k, \log(2k/\delta)\}$, then $\Omega$ is a $(1-\eps, \delta, k)$-OSI.    
\end{theorem}
\begin{proof}
We first note that $\EE\big[\Omega \Omega^T] = I_{n^d}$ follows directly from Lemma~\ref{lemma:isotropic}.
    Proving the OSI property \eqref{eq:osi} is equivalent to proving that 
    $$  
        1-\eps \leq \|\Omega^T U x\|_2^2, \quad \forall x \in \R^k, \ \|x\|_2=1,
    $$
    holds with probability at least $1-\delta$ for every fixed orthonormal matrix $U \in \R^{n^d \times k}$. In turn, this is equivalent to proving
    \begin{equation}
        \label{eq:osi-temp1}
        \lam_{\min}(Y) \geq 1-\eps, \quad
        Y = (\Omega^T U)^T \Omega^T U = \frac{1}{\ell} \sum_{i=1}^\ell \tilde{\omega}_i \tilde{\omega}_i^T,
    \end{equation}
    where $\tilde{\omega}_i = U^T \omega_{i}$. The vectors $\tilde{\omega}_i \in \R^k$ are independent, centered, and have the same distribution law as $\tilde{\omega} := U^T \omega$.
    We also have $\EE\big[\tilde{\omega} \tilde{\omega}^T] = U^T \EE\big[{\omega} {\omega}^T] U = I_k$.
    Theorem \ref{tm:moment} for $p=2$ gives
    \begin{equation}
        \label{eq:osi-temp2}        
        \EE[(\tilde{u}^T \tilde{\omega})^4]  
            = \EE[((U\tilde{u})^T \omega)^4] 
            \leq 3(1+2/r)^{d-1}, \quad \forall \tilde{u} \in \R^{k}, \|\tilde{u}\|_2=1.
    \end{equation}
    This concludes the proof; Theorem \ref{tm:tropp-mineig} combined with~\eqref{eq:osi-temp2} implies~\eqref{eq:osi-temp1} (and, hence, the desired OSI property)
    if the stated lower bound on $\ell$ holds. 
    The final claim of the theorem follows from $(1+\frac{2}{r})^{d-1} \leq (1+\frac{2}{d-1})^{d-1} \leq e^2$ for $r \geq d-1$.
\end{proof}

In the case $r \geq d-1$, the bound for the number of samples stated in Theorem~\ref{tm:osi} is only up to a constant $e^2$ worse than the bound that would be obtained from
Theorem~\ref{tm:tropp-mineig} for a matrix whose columns are unstructured standard Gaussian vectors.

\begin{remark}
Theorem~13 also clarifies the discussion of tensor-train adapted
sketches in~\cite[Table~1]{CazeauxDupuyJustiniano26}. In the notation
of~\cite{CazeauxDupuyJustiniano26}, let $R$ denote the TT rank and $P$ the number of
sampled TT vectors. The sketching map induced by $\Omega^T$ in this work corresponds to Gaussian TT random
projection $f_{\mathrm{TT}(R)}$ in~\cite{CazeauxDupuyJustiniano26}. Theorem~\ref{tm:osi} shows that, for
$R\geq d-1$, this map is a $(1-\varepsilon,\delta,k)$-OSI as soon as
$P = \redvel(
\varepsilon^{-2}
\max\{k,\log(2k/\delta)\}
).$
Thus, compared with the TTStack sketch of~\cite{CazeauxDupuyJustiniano26},
the sketch dimension is smaller by a factor of order $d$. 

Let us also briefly comment on the corresponding oblivious subspace
embedding property, where one requires both lower and upper norm
preservation. A standard $\varepsilon$-net argument, combined with the
moment bound of Theorem~\ref{tm:moment} and Proposition~\ref{prop:Lp-mean-bound} applied to the centered
variables
$(u^T\omega)^2 - 1$, with $u\in U$, $\|u\|_2=1$,
suggests that the same sketch dimension
$
P =
\mathcal O(
\varepsilon^{-2}
(k+\log(1/\delta))
)
$
should suffice for an oblivious subspace embedding, provided the TT rank
is increased to
$
R \gtrsim d\,(k+\log(1/\delta)).$
Indeed, with such a rank choice, the $L_p$-norm inflation factor
$\exp(p(d-1)/R)$ obtained in Theorem~\ref{tm:moment} remains bounded for moment orders
$p=\redvel(k+\log(1/\delta))$, which are the orders naturally required after
taking a union bound over a net of the unit sphere in a $k$-dimensional
subspace. We do not pursue this here, since it would essentially match
the parameter regime already obtained for TTStack in
\cite{CazeauxDupuyJustiniano26}, rather than improving it.
\end{remark}

\paragraph{Stochastic trace estimation with Nystr\"{o}m++.}

OSIs can be used to compute (randomized) low-rank approximations with probabilistic guarantees~\cite{CamanoEpperlyMeyerTropp25}. 
We recall a slightly rephrased result~\cite[Corollary 2.5]{CamanoEpperlyMeyerTropp25} when using the Nystr\"{o}m approximation.
Given a symmetric positive semidefinite matrix $A \in \R^{N \times N}$ and a random matrix $\Omega \in \R^{N \times \ell}$, $\ell \ll N$, the Nystr\"{o}m approximation is computed as
\begin{equation}
    \label{eq:nystrom}
    \hat{A} = X(\Omega^T X)^\pinv X^T, \quad \text{where } X = A\Omega.
\end{equation}
The matrix $\hat{A}$ is a rank-$\ell$ approximation of $A$. It is well known that the residual $A-\hat{A}$ is again symmetric positive semidefinite. Hence, the
nuclear norm $\|A-\hat{A}\|_\ast$ coincides with the trace of $A-\hat{A}$.

The following consequence of the proof of~\cite[Corollary 2.5]{CamanoEpperlyMeyerTropp25} will be used.
\begin{theorem}
    \label{tm:nystrom-approx}
    Consider symmetric positive semidefinite $A \in \R^{N \times N}$ and $k \leq N$. For an ($1-\eta$, $\delta$, $k$)-OSI $\Omega \in \R^{N \times \ell_1}$,
    construct the Nystr\"{o}m approximation $\hat{A}$ defined in~\eqref{eq:nystrom}.
    Then, for $M > 1$ and $0 < \eta,\delta < 1$, the inequality 
    $$
        \|A - \hat{A}\|_\ast \leq \Big(1 + \frac{M}{1-\eta}\Big) \cdot \|A - A_k\|_\ast,
    $$
    holds with probability at least $1-\frac{1}{M} - \delta$,
    where $A_k$ denotes a best rank-$k$ approximation of $A$ in the nuclear norm.
\end{theorem}

Equipped with a Nystr\"{o}m approximation $\hat{A} = X(\Omega^T X)^\pinv X^T$, $X = A\Omega$, and by using the cyclic property of the trace, one can approximate $\trace(A)$ as
\begin{align*}
    \trace(A)
        &= \trace(\hat{A}) + \trace(A - \hat{A})
        = \trace((\Omega^T X)^\pinv X^T X) + \trace(A - \hat{A}) \\
        &\approx \trace((\Omega^T X)^\pinv X^T X) + H_{\ell_2}(A - \hat{A}) \\
        &= \trace((\Omega^T X)^\pinv X^T X) + \frac{1}{\ell_2}( \trace(\Psi^T Y) - \trace(\Psi^T X(\Omega^T X)^\pinv (X^T \Psi) ) ) \\
        &=: \trace^{\text{N++}}_{\ell_1, \ell_2}(A).
\end{align*}
Here the $\ell_2$ samples used to compute the Girard--Hutchinson estimate $H_{\ell_2}(\Delta)$ of the matrix $\Delta := A - \hat{A}$ have been arranged as columns of the matrix $\Psi \in \R^{N \times \ell_2}$, and $Y = A\Psi$. This is the Nystr\"{o}m++ trace estimator \cite{PerssonCortinovisKressner22}; see also Algorithm \ref{alg:nystromPP}. 

Let us now fix the target probability for the Nystr\"{o}m++ estimate to, e.g., $\delta = 1/10$, and let the target relative accuracy be any $\eps > 0$.
Set the ranks of the Gaussian random TT vectors to $r \geq d-1$. Since $A = \hat{A} + \Delta$, and $\trace^{\text{N++}}_{\ell_1, \ell_2}(A) = \trace(\hat{A}) + H_{\ell_2}(\Delta)$, Theorem~\ref{tm:hutch-4th-moment} (conditioned on $\Omega$) shows that, for given $\tilde{\eps} > 0$, the inequality
\begin{equation}
    \label{eq:nystrom-temp}    
    |\trace^{\text{N++}}_{\ell_1, \ell_2}(A) - \trace(A)|
        = |H_{\ell_2}(\Delta) - \trace(\Delta)|
        \leq \tilde{\eps} \, \trace(\Delta)
\end{equation}
holds 
with probability at least $1-\tilde{\delta}$, provided that $\ell_2 = \redvel(\tilde{\eps}^{-2} \tilde{\delta}^{-1})$. We set $\tilde{\delta} = 1/20$ and fix $\tilde{\eps}$ later. Using that  
$
    \trace(\Delta) = \|A - \hat{A}\|_\ast,
$
we apply Theorem \ref{tm:nystrom-approx} with $M = 40$, $\eta = 1/2$, and $\delta=1/40$. Therefore, with probability at least $1-1/20$, we have $\trace(\Delta) \leq 81 \|A - A_k\|_\ast$ assuming that $\ell_1 = \redvel(k)$ (using Theorem \ref{tm:osi}), with $k$ specified later on. Combined with \eqref{eq:nystrom-temp} through the union bound, we have that
\begin{equation}
    \label{eq:nystrom-temp-1}    
    |\trace^{\text{N++}}_{\ell_1, \ell_2}(A) - \trace(A)|
        \leq 81 \tilde{\eps} \|A-A_k\|_\ast
\end{equation}
holds with probability at least $1-1/10$.
To connect $\|A-A_k\|_\ast$ with $\trace(A)$, we rewrite
\begin{equation}
    \label{eq:nystrom-temp-2}    
    \|A-A_k\|_\ast
        = \tau_k(A)\cdot \trace(A), \quad \text{with}\quad 
\tau_k(A)
:=
\frac{\|A-A_k\|_*}{\operatorname{tr}(A)}
=
\frac{\sum_{j>k}\lambda_j(A)}{\sum_j \lambda_j(A)},
\end{equation}
where $\lambda_j(\cdot)$ denotes the $j$th largest eigenvalue of a matrix.
The factor $\tau_k(A)$ can be viewed as a relative spectral tail mass;
it obviously satisfies $0 \leq \tau_k(A) \leq 1$ but a more precise estimate depends on the eigenvalue decay of the residual $A-A_k$.

Combining \eqref{eq:nystrom-temp-1} and \eqref{eq:nystrom-temp-2}, we obtain
$$
    |\trace^{\text{N++}}_{\ell_1, \ell_2}(A)-\trace(A)|
    \leq 81\tilde{\eps}\,\tau_k(A)\trace(A).
$$
Assuming $\tau_k(A)>0$, we choose
$
    \tilde{\eps} = \eps / (81\,\tau_k(A)).
$
Then
$$
    |\trace^{\text{N++}}_{\ell_1, \ell_2}(A)-\trace(A)|
        \leq \eps\,\trace(A)
$$
holds with probability at least $9/10$. In this case,
$\ell_1 = \redvel(k)$,
$\ell_2 = \redvel(\eps^{-2}\tau_k^2(A))$.
In particular, setting $k=\lceil 1/\eps\rceil$ and assuming
$\tau_k(A)=\redvel(\sqrt{\eps})$, then
$\ell_1=\redvel(\eps^{-1})$,
$\ell_2=\redvel(\eps^{-1})$.
In other words, we obtain $\redvel(\eps^{-1})$ complexity if the leading $\redvel(\eps^{-1})$ eigenvalues capture
all but an $\redvel(\sqrt{\eps})$ 
fraction of the trace.

While the dependence on the number of samples is often reduced from $\redvel(\eps^{-2})$ to $\redvel(\eps^{-1})$, our results indicate that the Nystr\"om++ algorithm will not be effective with random TT vectors unless there is some eigenvalue decay, that is, $\tau_k(A)$ remains small. The following discussion suggests that this requirement is not an artifact  of our analysis,
but it is tied to the absence of a uniform Frobenius-scale variance bound for low-rank Gaussian TT vectors.
More precisely, Theorem \ref{tm:hutch-4th-moment} ensures that 
$$
    |H_{\ell}(A) - \trace(A)| \leq \eps \trace(A),
$$
holds with high probability, where the number of samples $\ell$ does not depend on properties of $A$. Contrast this with \cite[Theorem 5]{CortinovisKressner22}, where the same sample complexity for unstructured Gaussian random vectors achieves
\begin{equation}
    \label{eq:trace-err-fnorm}    
    |H_{\ell}(A) - \trace(A)| \leq \eps \|A\|_F.
\end{equation}

This Frobenius-scale bound is the key ingredient in the standard
analysis of Nystr\"om++ with unstructured Gaussian vectors; it allows one
to control the residual trace estimator in terms of $\|A-\hat A\|_F$
rather than $\trace(A-\hat A)$. In contrast, Theorem~\ref{tm:hutch-4th-moment} only gives a
trace-scale bound for Gaussian TT vectors. This difference is not merely
an artifact of the proof. A uniform Frobenius-scale variance estimate of
the form
$
\operatorname{Var}[\omega^T A\omega]
\leq
C\|A\|_F^2
$
with a constant $C$ independent of $n$ and $d$ cannot hold for Gaussian TT
vectors of moderate rank.
Indeed, take $A=I_N$ with $N=n^d$. Let $G=\Omega^{[1]}\in\mathbb R^{n\times r}$
denote the first TT core of $\omega$. Conditioning on $G$ and averaging
over the remaining TT cores gives
$$
    \mathbb E\bigl[\|\omega\|_2^2\,\big|\,G\bigr]
        = \frac{n^{d-1}}{r}\|G\|_F^2.
$$
Since $\|G\|_F^2$ is chi-squared with $nr$ degrees of freedom, the law of
total variance yields
$$
    \operatorname{Var}[\omega^T I_N\omega]
        = \operatorname{Var}[\|\omega\|_2^2]
        \geq \operatorname{Var}!\left( \mathbb E\bigl[\|\omega\|_2^2\,\big|\,G\bigr]\right) 
        = \frac{n^{2d-2}}{r^2}\operatorname{Var}(\|G\|_F^2)
        = \frac{2n^{2d-1}}{r}.
$$
On the other hand, $\|I_N\|_F^2=N=n^d$. Hence
$$
    \frac{\operatorname{Var}[\omega^T I_N\omega]}{\|I_N\|_F^2}
        \geq \frac{2n^{d-1}}{r}.
$$
This ratio grows rapidly with the tensor order, unless $r$ is chosen exponentially large. Consequently, some dependence on the spectral tail, such as
the relative tail mass $\tau_k(A)$ appearing above, is natural.

\section{Conclusions}

We have shown that Gaussian random TT vectors are effective in stochastic trace estimation. In contrast to random rank-one tensor or Khatri--Rao test vectors, whose sample complexity can grow exponentially with the tensor order, Gaussian TT vectors recover dimension-independent trace-estimation guarantees once the TT rank is chosen at least proportional to the tensor order. In particular, with rank $r\geq d-1$, a median-of-means Girard--Hutchinson estimator achieves the same dependence on $\varepsilon$ and $\delta$ as the classical estimator with unstructured Gaussian vectors.

We have also established an oblivious subspace injection result for sketches formed from independent Gaussian TT vectors. For TT rank $r\geq d-1$, only
$
\mathcal O\!\left(\varepsilon^{-2}(k+\log(1/\delta))\right)
$
samples are sufficient for a $k$-dimensional target subspace. This improves the sketch dimension required by TTStack-type constructions by a factor of order $d$, while retaining a TT rank that grows only linearly with the tensor order.
Finally, we investigated the use of Gaussian TT sketches within Nystr\"om++ trace estimation. The resulting estimator can achieve the desired $O(\varepsilon^{-1})$ sample complexity under a natural spectral-tail condition.

The present work is theoretical in nature and does not address the design of robust and efficient implementations. Developing practical algorithms that fully exploit the TT structure while controlling rank growth, rounding errors, and computational overhead requires significant additional work and is beyond the scope of this paper.

\section{Funding}
The work of ZB and HO was supported by the Croatian Science Foundation grant IP-2025-02-3733 (``Data driven identification and dimension reduction of dynamical systems''). ZB was also supported by the project ``Implementation of cutting-edge research and its application as part of the Scientific Center of Excellence for Quantum and Complex Systems, and Representations of Lie Algebras'', Grant No.~PK.1.1.10.0004, co-financed by the European Union through the European Regional Development Fund - Competitiveness and Cohesion Programme 2021--2027.

\bibliographystyle{plain}
\bibliography{references}

\end{document}